\documentclass{article}

\usepackage{soul}
\usepackage{url}
\usepackage{xcolor}
\usepackage[hidelinks]{hyperref}
\usepackage[utf8]{inputenc}
\usepackage[small]{caption}
\usepackage{graphicx}
\usepackage{amsmath}
\usepackage{amsthm}
\usepackage{booktabs}
\usepackage{algorithm}
\usepackage{algorithmic}

\usepackage[utf8]{inputenc}
\usepackage[english]{babel}

\usepackage{amsmath}
\usepackage{graphicx}
\usepackage{hyperref}
\usepackage{amsthm}
\usepackage{amssymb} 
\usepackage{cancel} 
\usepackage{comment} 
\usepackage{latexsym} 

\usepackage{blkarray, bigstrut}
\usepackage{subcaption} 

\graphicspath{{Figures/}} 

\newcommand\R{\mathbb{R}}\newcommand\CC{\mathbb{C}}\newcommand\E{\mathbb{E}}

\newcommand{\patri}{\raisebox{5pt}{\rotatebox{180}{$\in$}}} 

\newtheorem{theorem}{Theorem}

\newtheorem{defn}[theorem]{Definition}

\newtheorem{example}[theorem]{Example}

\title{Sample Complexity of Bias Detection with Subsampled Point-to-Subspace Distances}
\author{German Martinez Matilla and Jakub Marecek}

\begin{document}
\maketitle

\begin{abstract}
Sample complexity of bias estimation is a lower bound on the runtime of any bias detection method.
Many regulatory frameworks require the bias to be tested for all subgroups, whose number grows exponentially with the number of protected attributes. 
Unless one wishes to run a bias detection with a doubly-exponential run-time, one should like to have polynomial complexity of bias detection for a single subgroup.
At the same time, the reference data may be based on surveys, and thus come with non-trivial uncertainty. 
Here, we reformulate bias detection as a point-to-subspace problem on the space of measures and show that, for supremum norm, it can be subsampled efficiently. In particular, our probabilistically approximately correct (PAC) results are corroborated by tests on well-known instances.
\end{abstract}

\section{Introduction}\label{sec:intro}

Regulatory frameworks, such as the AI Act in Europe, suggest that one needs to measure data quality, including bias in the data, 
as well as bias in the output of the AI system.
Basically, one could imagine bias detection as a two-sample problem in statistics, where given two sets of samples, one asks whether they come from the same distribution.
In practice, the two sets of samples often do not come from the same distribution, but one would like to come with an estimate of the distance between the two distributions. 
The distance estimate, as any other statistical estimate \cite{tsybakov2009nonparametric}, comes with an error.
Clearly, one would like the error in the estimate to be much smaller than the estimated value
in order for the bias detection to be credible, stand up in any court proceedings, etc.  

Sample complexity is the  number of samples that makes it possible to estimate a quantity to a given error.
A lower bound on sample complexity then suggests the largest known number of samples, universally required to reach a given error. 
The sample complexity of bias estimation depends on the distance between the distributions (measures) used, including the Wasserstein-1 \cite{vaserstein1969markov}, Wasserstein-2 \cite{vaserstein1969markov,dudley1969speed}, Maximum Mean Discrepancy (MMD, \cite{gretton2012kernel}), Total Variation (TV), operator infinity norm \cite{wolfer2021statistical}, Hellinger distance \cite{hellinger1909neue} also known as Jeffreys distance, and a variety of divergences, including Kullback–Leibler (KL) and Sinkhorn. 
Notice that the TV distance (also known as the statistical distance) can be related to the KL divergence via the Pinsker inequality.
Throughout, the accuracy increases with the number of samples taken, 
but the rate depends on the dimension. As is often the case in high-dimensional probability, 
the ``curse of dimensionality'' suggests that the number of samples for a given error grows exponentially with the dimension. 
This has recently been proven for 
Wasserstein-1 \cite{fournier2015rate,weed2019sharp,panaretos2019statistical}, Wasserstein-2 \cite{dudley1969speed,fournier2015rate,weed2019sharp,panaretos2019statistical}, Wasserstein-$\infty$ \cite{fournier2015rate,liu2018rate,weed2019sharp,panaretos2019statistical}, TV \cite{devroye2018total,wolfer2021statistical,arbas2023polynomial}, operator infinity norm \cite{wolfer2021statistical}, and a variety of divergences including Sinkhorn \cite{genevay2019sample,quang2021convergence}.
For others, such as Hellinger and Jeffreys, it follows from their relationship to TV distance. 
For MMD, the situation is more complicated: while the original paper \cite{gretton2012kernel} claimed polynomial sample complexity, the more recent work explains the dependence on dimension \cite{NIPS2016_5055cbf4}, even under assumptions about smoothness coming from applying smooth kernels in a probabilistically approximately correct (PAC) setting.
More broadly, while sample complexity may be lower under assumptions on the smoothness of the measure and certain invariance properties \cite{chen2023sample,tahmasebisample}, it is very hard to assume that those assumptions hold in real-world data.  

Sample complexity of bias estimate is important for a number of reasons:
first, the sample complexity is a lower bound on the runtime, even in cases, where this is decidable  \cite{lee2023computability}.
First, many regulatory frameworks require that bias be tested for all subgroups, of which there may be exponentially many in the number of protected attributes. 
Unless one wishes to run a bias detection with a doubly-exponential run-time, one should like to have polynomial (or even sublinear) complexity of the bias detection for a single subgroup. 

Here, we reformulate the problem as a point-to-subspace problem on the space of measures and show that it can be subsampled efficiently. 

\subsection{Notation}
We use the following notation:
\begin{align*}
    \Sigma_n&\triangleq\{a\in\R^n_+:\sum^n_{i=1}a_i=1\}\\
    &\triangleq\text{Probability simplex with $n$ bins}\\
    \mathcal{X},\mathcal{Y}&\triangleq\text{Metric spaces equipped with a distance.}\\
    \mathcal{M}_+(\mathcal{X})&\triangleq\text{Set of all positive measures on $\mathcal{X}$}\\
    \mathcal{M}^1_+(\mathcal{X})&\triangleq\text{Set of probability measures, i.e., any} \\
    \alpha\in\mathcal{M}^1_+(\mathcal{X}) & \text{ is positive and has $\alpha(\mathcal{X})=\int_\mathcal{X}d\alpha =1.$}\\
    \delta_{\{x\}}&\triangleq\text{Dirac's delta at location $x.$} \\
\end{align*}





\section{A Motivating Example and Problem Definition}\label{sec:mot_example}

Let us consider a motivating example. In the COMPAS \cite{misc_ProPublica_COMPAS} dataset, there are $\approx 7\cdot 10^3$ instances, which surely do not account for the people sentenced using the COMPAS system, but we take to be a representative sample.
Focusing on the attribute \emph{decile\_score}, which tries to predict the recidivism risk, we can consider the resulting histogram as a point in $\R_+^{10},$ or, considering the normalised histogram, as a point in the 10 bin simplex, $\Sigma_{10}$ (cf. Figure \ref{fig:Whole_COMPAS_d_score}). More explicitly, Figure \ref{fig:Whole_COMPAS_d_score_counts} is a point $\R_+^{10}\patri h=(1440,941, \dots , 383),$ while for Figure \ref{fig:Whole_COMPAS_d_score_fraction} we have: $a\in\R_+^{10}:\sum_{i=1}^{10}a_i=1.$

\begin{figure}
    \centering
    \begin{subfigure}[b]{0.3\textwidth}
        \includegraphics[width=\textwidth]{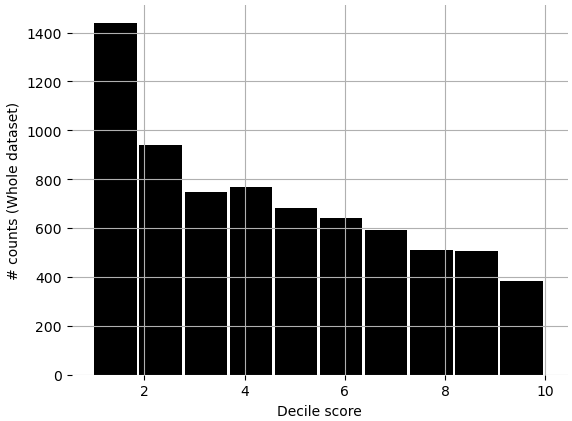}
        \caption{Vector \\ $\R_+^{10}\patri h=(1440, 941, \dots , 383)$} 
        \label{fig:Whole_COMPAS_d_score_counts}
    \end{subfigure}
    ~ 
    \begin{subfigure}[b]{0.3\textwidth}
        \includegraphics[width=\textwidth]{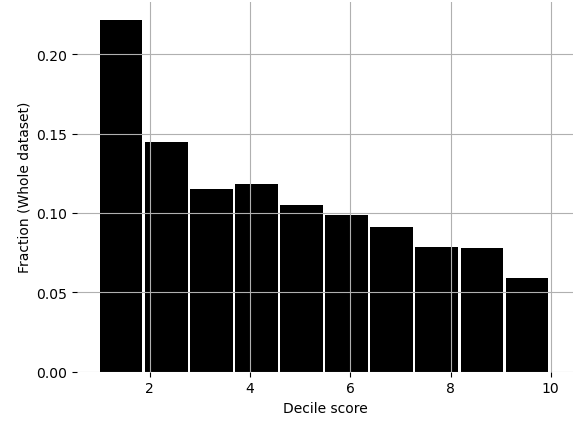}
        \caption{Probability vector $a\in\Sigma_{10}$ or measure $\alpha\in\mathcal{M}^1_+(\R)$}
        \label{fig:Whole_COMPAS_d_score_fraction}
    \end{subfigure}
    \caption{Whole COMPAS dataset by \emph{decile\_score}}\label{fig:Whole_COMPAS_d_score}
\end{figure}

It is widely understood \cite{bao2021s} that the COMPAS dataset captures a subset of the cases considered using the COMPAS system. We account for the sampling error by considering a constant uncertainty interval for every bin. 
That is, this histogram now consists of the many vectors $\tilde{h}^j\in\R_+^{10}$ that fit in the uncertainty set (composed of the product of uncertainty intervals): $\tilde{h}^j_i=([h_i-\Delta,h_i+\Delta])_i,\ i=1\dots 10.$ See Figure \ref{fig:Whole_pop_d_score_counts}. This vectors $\tilde{h}^j$ span a discrete subspace $D\subset\R_+^{10}.$
Alternatively, we can add the error $\Delta$ to the normalised histogram (Figure \ref{fig:Whole_pop_d_score_fraction}). We end up with an infinite number of points $\tilde{a}^j\in\Sigma_{10}:\tilde{a}^j_i=([a_i-\Delta,a_i+\Delta])_i, j=1, \dots , \infty$

\begin{figure}
    \centering
    \begin{subfigure}[b]{0.3\textwidth}
        \includegraphics[width=\textwidth]{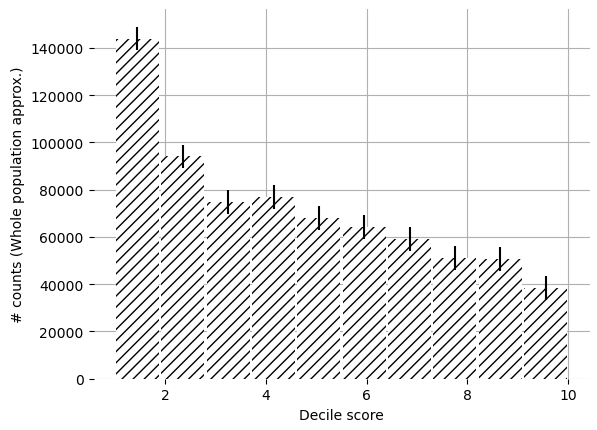}
        \caption{Finite vectors $\tilde{h}^j\in\R_+^{10}$ spanning discrete subspace $D\subset\R_+^{10}$}
        \label{fig:Whole_pop_d_score_counts}
    \end{subfigure}
    ~ 
    \begin{subfigure}[b]{0.3\textwidth}
        \includegraphics[width=\textwidth]{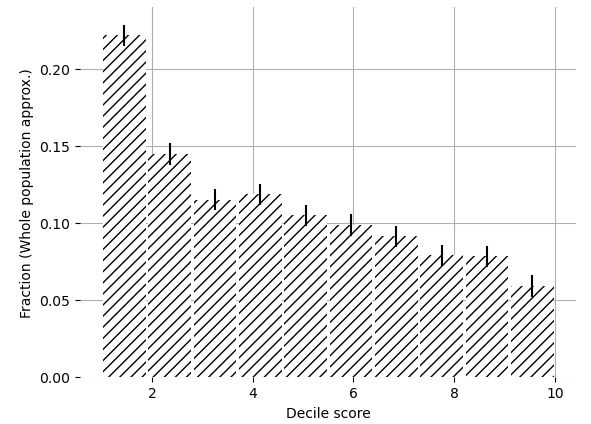}
        \caption{Infinite histograms $\tilde{a}^j\in\Sigma_{10}$ or subspace $V\subset\mathcal{M}^1_+(\R)$}
        \label{fig:Whole_pop_d_score_fraction}
    \end{subfigure}
    \caption{Whole population approximation by \emph{decile\_score}}\label{fig:Whole_pop_d_score}
\end{figure}

Next, we look at the normalised histograms in the light of measure theory. The 10 different coordinates corresponding to the 10 bins of the histogram are considered in this fashion as 10 different points, $x_i,$ on the real line. We have then for Figure \ref{fig:Whole_COMPAS_d_score_fraction}, a single atomic (i.e., discrete) measure $\mathcal{M}^1_+(\R)\patri\alpha=\sum_{i=1}^{10}a_i\delta_{x_i}\ ,\ \sum_{i=1}^{10}a_i=1.$ For Figure \ref{fig:Whole_pop_d_score_fraction}, we have an infinite number of such measures:

\[
\left\{\begin{array}{lll}
\alpha_1=\sum_{i=1}^{10}\tilde{a}^1_i\delta_{x_i}  \qquad & \tilde{a}^j_i=([a_i-\Delta,a_i+\Delta])_i \\
\vdots \qquad \qquad \qquad \qquad \quad ; & \\
\alpha_\infty=\sum_{i=1}^{10}\tilde{a}^\infty_i\delta_{x_i} \qquad & \sum_{i=1}^{10}\tilde{a}^j_i=1, j=1,\dots , \infty
\end{array}\right.
\]

\

They form a subspace $V\subset\mathcal{M}^1_+(\R).$
In this previous setting, our problem will consist in determining whether a certain subgroup in the data set (e.g., Figure \ref{fig:Compas_d_score_AA_fraction}) follows the same distribution in the general population within a range. In our previous notation, this is the same as the test if measure $\alpha_0$ is in $S.$

Formally, we can phrase this as a point-to-subspace query in the $\ell_\infty$ distance on the space of measures:



\begin{figure}[t!]
\centering
\includegraphics[scale=0.4]{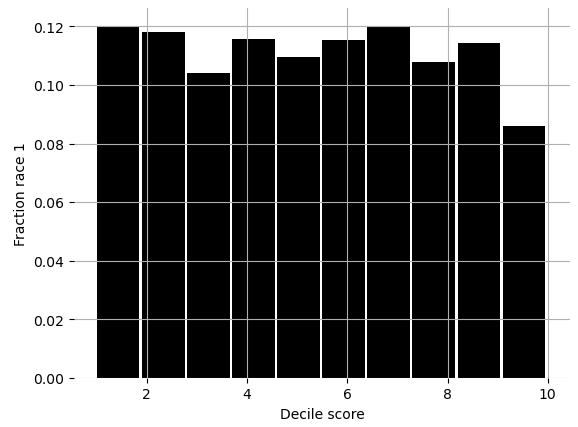}
\caption{Test measure, $\alpha_0$}
\label{fig:Compas_d_score_AA_fraction}
\end{figure}


\begin{algorithm}[tb]
    \caption{Point-to-subspace query in the supremum norm}
    \label{alg:mot_example}
    \textbf{Input}: Test measure $\alpha_0\in\mathcal{M}^1_+(\R),$ histogram $h\in\R^{10}_+$ 
    
    \textbf{Parameter}: $\Delta\in\R.$
    
    \textbf{Output}: True/False 
    \begin{algorithmic}[1]
        \STATE $a \gets \text{normalise }h$. 
        \FOR{i=1 to 10} 
        \IF{$|\alpha_0(x_i)-a_i|\geq\Delta$}
        \STATE \textbf{return}  FALSE
        \ENDIF
        \ENDFOR
        \STATE  \textbf{return} TRUE
    \end{algorithmic}
\end{algorithm}

The for loop calculates the $\ell_\infty$ norm between the test measure and our histogram approximation and compares it to the threshold $\Delta$, that is, it checks whether $\max_i|\alpha_0(x_i)-a_i|\lessgtr\Delta.$ Then the output corresponds to:
\[
\begin{array}{cc}
\text{TRUE } \leftrightarrow  \max_i|\alpha_0(x_i)-a_i| < \Delta \leftrightarrow\ \alpha_0\in V\subset\mathcal{M}^1_+(\R) \\
\text{FALSE} \leftrightarrow \max_i|\alpha_0(x_i)-a_i|\geq\Delta \leftrightarrow\ \alpha_0\notin V\subset\mathcal{M}^1_+(\R) \\
\end{array}
\]

\section{Further Definitions}




In this section, we introduce briefly the metrics between distributions mentioned in the Introduction \ref{sec:intro}. As we do not have direct access to the distributions $\alpha(x),\beta(y),$ we are forced to rely on estimators based on samples thereof. 


\subsection{Maximum Mean Discrepancy (MMD)}

This distance is defined as 
the maximum difference between the expected value of a function, $f,$ of $x\in\mathcal{X},$ and its counterpart for $y\in\mathcal{Y},$ the random variable of the distribution(measure) $\beta$, i.e.:

\begin{align}
\text{MMD}[\alpha,\beta]&=\sup_f\big[\E_{x\sim\alpha}[f(x)]-\E_{x\sim\beta}[f(y)]\big]\nonumber\\
&\equiv\sup_f\big[\int_{\mathcal{X}}f(x)\text{d}\alpha-\int_{\mathcal{X}}f(y)\text{d}\beta\big].
\end{align}

The rationale behind this definition is the fact that, if the function space where the witness function, $f,$ lives is big enough, we will have an accurate two-sample test. That is, we can ascertain wether $\alpha=\beta$ given that $\text{MMD}[\alpha,\beta]=0.$ This is guaranteed for the space of continuous bounded functions, $\mathcal{C}_b,$ cf. Lemma 1 \cite{gretton2008kernelmethodtwosampleproblem} (Lemma 9.3.2 \cite{MR1932358}). In practice, however, a different, more manageable function space, yet with the same guarantee, is used to work in: a Reproducing Kernel Hilbert Space (RKHS). A RKHS is a regular Hilbert space with a 

\begin{defn}[Reproducing kernel] \cite{MR2239907}

Let $E\neq\emptyset,$ be an abstract set. A reproducing kernel is a function 
\begin{align}\nonumber
k:E\times E &\to \CC \\ \nonumber
(s,t) &\mapsto k(s,t)
\end{align}
\[
\text{s.t. } 
\underbrace{\left\{ 
\begin{array}{ll}
\forall t\in E &, k(\cdot,t)\in\mathcal{H}\\
\forall t\in E, \forall\varphi\in\mathcal{H} &, \langle\varphi,k(\cdot,t)\rangle=\varphi(t)
\end{array}
\right\}}_{\Rightarrow \forall (s,t)\in E\times E, k(s,t)=\langle k(\cdot,t),k(\cdot,s)\rangle}
\]

\end{defn}

Where we have used the notation $\left\{\begin{array}{ll} k(\cdot,t) \\ k(s,\cdot) \end{array}\right\}$ to refer to the mappings 
$\left\{
\begin{array}{ll}
s\mapsto k(s,t)\text{ with fixed }t \\
t\mapsto k(s,t)\text{ with fixed }s
\end{array}
\right\}$

\begin{defn}[MMD in RKHS]

In particular, the MMD can be defined on a unit ball in a universal (=dense in $\mathcal{C}_b$ in the $\infty$ norm) RKHS:

\begin{equation}
\text{MMD}[\alpha,\beta]=\sup_{\|f\|_{\mathcal{H}}\leq 1}\big[\int_{\mathcal{X}}f(x)\text{d}\alpha-\int_{\mathcal{X}}f(y)\text{d}\beta\big]\ ,\ \forall f\in\mathcal{H}
\end{equation}

\end{defn}

\subsection{Wasserstein distance}

The origin of the Wasserstein distance {\cite{peyre2020computationaloptimaltransport}} can be traced back to the work of Monge on Optimal Transportation (OT) \cite{monge1781memoire}. For the discrete case, we consider discrete measures $\alpha\in\mathcal{M}(\mathcal{X})$ and $\beta\in\mathcal{M}(\mathcal{Y}),$ or equivalently, simplices $a\in\Sigma_n$ and $b\in\Sigma_m.$ Given a map (the Monge map):
\begin{align}\nonumber
T: \mathcal{X} &\to \mathcal{Y} \\ \nonumber
(x_1,\dots,x_n) &\mapsto (y_1,\dots,y_m)
\end{align}

that verifies \begin{equation}\label{eq:mass_conserv}b_j=\sum_{i:T(x_i)=y_j}a_i\ ,\ \forall j\in 1,\dots,m,\end{equation} the problem seeks to minimise 
some transportation cost, parameterized by $c(x, y)\in\R^{n\times m}$ defined for $(x, y) \in \mathcal{X} \times \mathcal{Y}.$ Putting all together, Monge's discrete OT problem reads:

\begin{equation}\label{eq:disc_mong}
\begin{aligned}
\min_T \quad & \sum_{i} c\big(x_i,T(x_i)\big)\\
\textrm{s.t.} \quad & \sum_{i:T(x_i)=y_j}a_i=b_j\\
\end{aligned} \ \ ; \begin{array}{ll}\ c\in\R^{n\times m}\ ,\ a,\ b\in\Sigma_n,\ \Sigma_m \\\forall i\in 1,\dots,n \ ,\ \forall j\in 1,\dots,m 
\end{array}
\end{equation}

For the continuous problem, we have:
\begin{equation}\label{eq:cont_mong}
\begin{aligned}
\min_T \quad & \int_\mathcal{X} c\big(x,T(x)\big)d\alpha(x)\\
\textrm{s.t.} \quad & \int_{T^{-1}(B\subset\mathcal{Y})}d\alpha(x) = \beta(B\subset\mathcal{Y})\\
\end{aligned} ; \begin{array}{ll}
\alpha\in\mathcal{M}^1_+(\mathcal{X})\\
\beta\in\mathcal{M}^1_+(\mathcal{Y})
\end{array}
\end{equation}

The solution of this optimisation problem, gives as a result the most efficient way of transporting a pile of sand into a hole in the ground. It is therefore named sometimes as the earth mover's distance. This original formulation has its limitations: In the discrete case, \eqref{eq:disc_mong}, there might be no Monge map possible, while the continuous problem, \eqref{eq:cont_mong}, is not convex.

To overcome this difficulties, a relaxed version of the problem is proposed by Kantorovich in \cite{Kantorovich2006}. In this setting, two marginal conditions need to be fulfilled, instead of the mass conservation equation \eqref{eq:mass_conserv}. The Monge map is substituted by the coupling $P\in\R^{n\times m}$ between two probability vectors in $\Sigma_n,\Sigma_m$ in the discrete case:
\begin{equation}\label{eq:disc_kant}
\begin{aligned}
\min_P \quad & \sum_{ij} c_{ij} P_{ij}\\
\textrm{s.t.} \quad & \sum_j^m P_{ij} = a_i\\
  &\sum_i^n P_{ij} = b_j    \\
\end{aligned}\quad ,\ c, P\in\R_+^{n\times m}\ ,\ a,\ b\in\Sigma_n,\ \Sigma_m 
\end{equation}


In the continuous case, we have a coupling $\pi\in\mathcal{M}^1_+(\mathcal{X},\mathcal{Y})$ between two measures. Finally, the p-Wasserstein distance, $W_p(\alpha,\beta),\ p\in [ 1,\infty)$, is defined as: 

\begin{equation}\label{eq:cont_kant}
\begin{aligned}
W_p(\alpha,\beta) = \min_\pi \quad & \Big( \int_{\mathcal{X}\times\mathcal{Y}} c^p(x,y) d\pi(x,y) \Big)^{\frac{1}{p}}\\
\textrm{s.t.} \quad & \int_{A\times\mathcal{Y}} d\pi(x,y) = \alpha(A)\\
  &\int_{\mathcal{X}\times B} d\pi(x,y) = \beta(B)    \\
\end{aligned}
\end{equation}

where

\[
\pi\in\mathcal{M}^1_+(\mathcal{X},\mathcal{Y})\ ,\ \left\{\begin{array}{ll}\alpha\in\mathcal{M}^1_+(\mathcal{X})\ ,\ A\subset\mathcal{X} \\ \beta\in\mathcal{M}^1_+(\mathcal{Y})\ ,\  B\subset\mathcal{Y}\end{array}\right.
\]

\section{A Subsampling Scheme}
\label{sec:subsampling}


Having introduced the problem in Section \ref{sec:mot_example}, we generalise upon it. We will take into account several numerical attributes of a dataset rather than just one, as we exemplified with the attribute \emph{decile\_score} in the COMPAS dataset. Categorical attributes may also be included if properly one-hot encoded. In this fashion, we call all the dataset attributes considered, encoded features, and denote them by $f_1,\dots, f_n.$ We believe it is worth consider in more detail the two-dimensional case, essentially because it is plottable in ordinary three-dimensional space:

\begin{example}

Let us continue the COMPAS example with encoded features 
$f_1=decile\_score, f_2=age.$ The equivalent of Figure \ref{fig:Whole_pop_d_score_fraction} will be the infinitely many $2$-dimensional histograms, in 3D (cf. Figure \ref{fig:Compas_3D}),
given by the discrete measures:

\[
\left.\begin{array}{lll}
\alpha_1=\sum_{i=1}^{10}\sum_{j=1}^{10}\tilde{a}^1_{ij}\delta_{\{x_i,y_j\}} \\
\vdots \\ 
\alpha_\infty=\sum_{i=1}^{10}\sum_{j=1}^{10}\tilde{a}^\infty_{ij}\delta_{\{x_i,y_j\}} 
\end{array}\right\}\in V\subset\mathcal{M}^1_+(\R^2) 
\]

\

where
\[
\Sigma_{10\times 10}\patri\tilde{a}^k_{ij}=([a_{ij}-\Delta,a_{ij}+\Delta])_{ij}\ ,\ \left\{\begin{array}{lll}
k=1,\dots , \infty\\
i=1,\dots , 10\\
j=1,\dots , 10
\end{array}\right.
\]

Upon which we will compare the test measure $\sum_{i=1}^{10}\sum_{j=1}^{10}a^0_{ij}\delta_{\{x_i,y_j\}}=\alpha^0\in\mathcal{M}^1_+(\R^2).$

\begin{figure}[t!]
\centering
\includegraphics[scale=0.5]{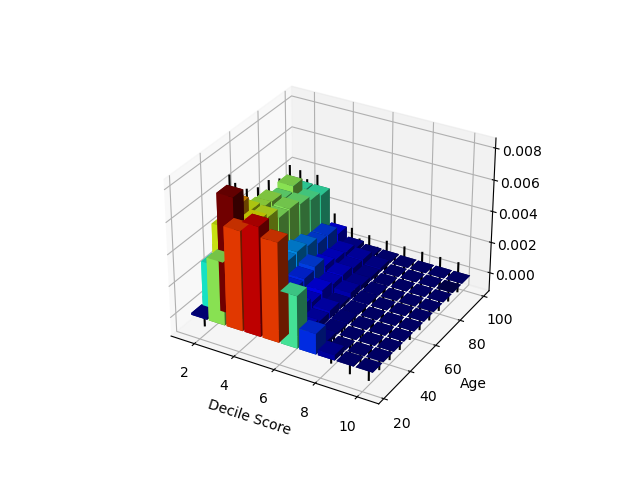}
\caption{decile\_score and age relative frequencies with error bars of length $\Delta$}
\label{fig:Compas_3D}
\end{figure}

\end{example}


As we have seen from the example, the number of hyperplanes $|\alpha^0(x_i, y_j)-a_{i j}|\leq\Delta$ scales very fast. Therefore, we consider a subsampling scheme, where the features are selected uniformly at random.






In order to implement this subsampling scheme, we fix a histogram approximation of the test measure and the subspace. 
In fact, this will depend on the binning of every attribute. Denote by $b_i$ the number of bins in the histogram of attribute $i.$ The infinite $n-$dimensional histograms will be given by:


\begin{equation}\label{eq:histogram_approx}
\left.\begin{array}{lll}
\alpha_1=\sum_{i,\dots ,z}^{b_1,\dots,b_n}\tilde{a}^1_{i\dots z}\delta_{\{x_i,\dots,w_z\}} \\
\vdots \\ 
\alpha_\infty=\sum_{i,\dots ,z}^{b_1,\dots,b_n}\tilde{a}^\infty_{i\dots z}\delta_{\{x_i,\dots,w_z\}}
\end{array}\right\}\in V\subset\mathcal{M}^1_+(\R^n) 
\end{equation}

where
\begin{equation}
\underbrace{\tilde{a}^i_{j\dots z}}_{\in\Sigma_{b_1\times\dots\times b_n}}=([a_{j\dots z}-\Delta,a_{j\dots z}+\Delta])_{j\dots z}\ 
\left\{\begin{array}{llllll}
i=1,\dots , \infty \\
j=1,\dots , b_1 \\
k=1,\dots , b_2 \\
\vdots \\
z=1,\dots,b_n
\end{array}\right.
\end{equation}

Upon which we will compare the test measure $\sum_{i,\dots ,z}^{b_1,\dots,b_n}a^0_{i\dots z}\delta_{\{x_i,\dots,w_z\}}=\alpha^0\in\mathcal{M}^1_+(\R^n).$ By checking the $N=\prod^{n}_{i=1}b_i$ inequalities:
\begin{equation}\label{eq:gen_case_compar}
|\alpha^0(x_j,\dots,w_z)-a_{j\dots z}|\leq\Delta
\end{equation}









Out of the total $N=\prod^{n}_{i=1}b_i$ points (=bins) in $\R^{n},$ denoted by the coordinates 
\[
(x_j,\dots ,w_z)\left\{\begin{array}{llll}
j=1,\dots , b_1 \\
k=1,\dots , b_2 \\
\vdots \\
z=1,\dots,b_n
\end{array}\right.,
\]

we will pick $s,$ uniformly at random, resulting in the $S\subset (x_i,\dots ,w_z),$ bins, with $|S|=s.$ Then, the projection of a discrete measure on these sampled bins is defined as:

\begin{defn}\label{def:proj}

The projection of a discrete measure, $\alpha=\sum_{i\dots z}a_{i\dots z}\delta_{\{x_i,\dots w_z\}},\alpha\in\mathcal{M}^1_+(\R^n)$ on $S\subset (x_i,\dots ,w_z), |S|=s,$ is the restriction
\[
\alpha |_S=\sum_{i\dots z}a_{i\dots z}\delta_{\{S\}}, \alpha|_S\in\mathcal{M}_+(\R^n)
\]

\end{defn}

The projection of the subspace $V\subset\mathcal{M}^1_+(\R^n)$ is obtained by projecting its constituent measures.

We then propose a subsampling approach in which the output of the algorithm is accurate up to some probability with some guarantees on the number of samples, $s.$
See Algorithm \ref{alg:general_case}. At first, the histogram approximation \eqref{eq:histogram_approx} of the subspace $V\subset\mathcal{M}^1_+(\R^n)$ is computed. 
We sample a subset $S$ of the $s$ bins of the histogram approximation uniformly at random.
We compute a projection $\alpha |_S$ of the histogram approximation of the test measure on the subset $S$ of bins, as in Definition \ref{def:proj}. 
We compute a projection $V |_S$ of the histogram approximation of the subspace on the subset $S$ of bins, by projecting its measures. 
The $\ell_0$ distance is calculated between the projections $\alpha |_S$ and $V |_S$. 

If the $\ell_0$ distance of the projections is not greater than $\Delta$, we estimate that the test measure is inside the subspace. 
The output is TRUE $\leftrightarrow \alpha^0\in V\subset\mathcal{M}^1_+(\R),$ although the estimate can be a false positive with a probability that we bound in Theorem \ref{thm:main} below.

If the $\ell_0$ distance of the projections is greater than the threshold $\Delta$, we know the original 
test measure $\alpha$ is at least $\Delta$ away from the original subspace $V$. 
We report FALSE $\leftrightarrow \alpha^0\notin V\subset\mathcal{M}^1_+(\R).$

The number of samples required for a determined certainty level is a matter of the next section.


\begin{algorithm}[tb]
    \caption{Subsampled point-to-subspace query in the supremum norm}\label{alg:general_case}
    
    \textbf{Input}: Number of samples taken independently uniformly at random, $s,$ Test measure $\alpha_0\in\mathcal{M}^1_+(\R^n),$  histograms $h_i\ ,\ i=1,\dots,n$
    
    \textbf{Parameter:} Threshold $\Delta\in\R$
    
    \textbf{Output}: True/False 
    \begin{algorithmic}[1]
        \STATE $a^i \gets \text{normalise }h_i\ ,\ i=1,\dots, n$ 
        \STATE Out of the $N$ bins in the $n$-dimensional histogram, choose  $S\subset (x_i,\dots ,w_z),$ bins, $|S|=s.$
        \FOR{bin $\in$ S} 
        \IF{$|\alpha^0(x_i,\dots,w_z)|_S-a_{i\dots z}|_S|\geq\Delta$} 
        \RETURN FALSE
        \ENDIF
        \ENDFOR
        \STATE \textbf{return} TRUE
    \end{algorithmic}
\end{algorithm}

\section{Main Result}

Our main result shows that the subsampling scheme of Section~\ref{sec:subsampling} is a probabilistically approximately correct estimator. 
Specifically, we bound from below the number of encoded features (coordinates) required to obtain a one-sided error of the estimate of the  $\ell_\infty$  point-to-subspace distance at a fixed probability level. We include first several definitions for completeness.

\begin{defn}[\cite{MR3408730}]\label{def:range_space}
A range space is a pair $(X,\mathcal{R}),$ $X$ being a set, and $\mathcal{R},$ a family of subsets of $X,$ $\mathcal{R}\subseteq 2^X.$
\end{defn}

\begin{defn}[\cite{MR3408730}]\label{def:shattering}
    Let $(X,\mathcal{R})$ be a range space and let $A\subset X$ be a finite set. If the set of all subsets of $A$ that can be obtained by intersecting $A$ with any range $\mathcal{R},$ $\Pi_\mathcal{R}(A),$ equals its power set, that is:
    \[\Pi_\mathcal{R}(A)=2^A,\]
    then we say that $A$ is shattered by $\mathcal{R}.$
\end{defn}

\begin{defn}[\cite{MR3408730}]\label{def:vc_dim}
The VC-dimension of a range space is the smallest integer $d$ such that no finite set $A\subset X$ of cardinality $d+1$ is shattered by $\mathcal{R}.$ If no such $d$ exists, the VC-dimension is infinite.
\end{defn}


\begin{theorem}\label{thm:main}
Consider a test measure $\alpha^0\in\mathcal{M}^1_+(\R^n),$ and
a subspace $V \in \mathcal{M}^1_+(\R^n)$,
such that
$\ell_0(v_{hist}(\alpha)) \le \epsilon N$,
where $v_{hist}(\alpha)$ is the vector of violations of the constraint $v(\alpha) := \{ \max\{ |\alpha^0(x_i)-a_i| - \Delta, 0 \}, i \in hist \}$.
In that case 
Algorithm \ref{alg:general_case}, taking $s$ independent samples uniformly at random, where 
\[
s=O\left(\frac{n\log n}{\epsilon}\log\frac{n\log n}{\epsilon}+\frac{1}{\epsilon}\log\frac{1}{\delta}\right),\ \epsilon,\delta\in(0,1)
\]

produces a false positive, that is, reports that the test measure, $\alpha^0,$ is in $S\subset\mathcal{M}^1_+(\R^n),$ while it is not, with probability $\delta.$
\end{theorem}

\begin{proof}(\emph{Sketch}) The proof technique is standard. 
Following the reasoning of \cite{Marecek_2022}, we identify as the range space, $X,$ (Definition \ref{def:range_space}) under study, the polyhedron in $\R^n$ delimited by the inequalities in \eqref{eq:gen_case_compar}, with ranges, $\mathcal{R}_1,\mathcal{R}_2$ given by these $2n$ half-spaces. We will then bound its VC dimension (Definition \ref{def:vc_dim}) and apply the $\epsilon-$net theorem (\cite{MR884223}\cite{10.1145/10515.10522}\cite{10.1145/12130.12173}\cite{MR884226}). In this light, assume that this range space has VC dimension $d$. If the number of independent, uniform, random samples $s$ is:
\[
s=O\left(\frac{d}{\epsilon}\log\frac{d}{\epsilon}+\frac{1}{\epsilon}\log\frac{1}{\delta}\right)\ ,\ \epsilon,\delta\in(0,1)
\]
We get an $\epsilon$-net (\cite{10.1145/10515.10522}) with probability at least $1-\delta.$ Therefore, Algorithm \ref{alg:general_case} errs with probability $\delta,$ for an input with $\epsilon N$ distances greater than the threshold, $\Delta,$ that is, by computing all the data, we will get $|\alpha^0(x_j,\dots,w_z)-a_{j\dots z}|\geq\Delta,$ exactly $\epsilon N$ times.
For an input in which there are no such, i.e., all the comparisons satisfy $|\alpha^0(x_j,\dots,w_z)-a_{j\dots z}|\leq\Delta,$ the algorithm does not err and returns TRUE always.

As already mentioned, we next prove that the VC dimension of the range spaces $(X,\mathcal{R}_1),$ and $(X,\mathcal{R}_2),$ is bounded above by $n+1$. The only difference between these ranges is the sign of the absolute value taken. We thus have $n$ half-spaces for each range. We choose one of the two, which we will denote for simplicity $(X,\mathcal{R}),$ and proceed identically for the second. From this $(X,\mathcal{R}),$ we can build the range space having as ranges not only the previous $n$ half-spaces given by the $n$ hyperplanes, but all possible half-spaces in $\R^n,\ (X,\mathcal{R}_n).$ We clearly have:
\[\text{VC dim}\big((X,\mathcal{R})\big)\leq\text{VC dim}\big((X,\mathcal{R}_n)\big)\]
The VC dimension of $(X,\mathcal{R}_n)$ is known to be $n+1.$ To demonstrate it, one has to prove both that:
(1) no more than $n+1$ points can be shattered (Definition\ref{def:shattering}).
(2) at least one set of $n+1$ points can be. 
Both are achieved in Section H.5 in \cite{Barba_Comp_Geom}, (1) by Radon's theorem. Additionally, (2) is the matter of Exercise 14.7 in \cite{MR3674428}.
Finally, by Theorem 14.5 \cite{MR3674428}, the union of the two range spaces has dimension $O(n\log n).$

\end{proof}

Note that there are no possible false negatives (cases in which Algorithm \ref{alg:general_case} claims that the test measure is outside the subspace while it truly belongs to it), simply because if there are no comparisons such that $|\alpha^0(x_j,\dots,w_z)-a_{j\dots z}|\leq\Delta,$ the algorithm will not be able to hit any such, no matter the number of samples taken.





\section{Experimental Results}






The experiments performed for this section have been implemented in Python. The Wassertein distances were computed using Python Optimal Transport, the module by \cite{flamary2021pot}. The code was run on a standard laptop equipped with an AMD Ryzen 7 CPU and 30 GB of RAM. The operating system was Debian GNU/Linux.

\subsection{The Data}

We used two sources of data: Adult data set \cite{misc_adult_2} and folktables \cite{ding2021retiring}, which we have used to retrieve data from the US census. Together with the widely used Adult data set, we deemed appropriate the inclusion of folktables, since the larger amount of data would showcase the advantage of the proposed method in a clearer way. It can actually be considered a superset of the Adult dataset. In both, the protected attribute studied was \emph{SEX}. 

Using folktables, we retrieved data from the US census from all fifty states available in the year 2018. In the case of Adult, we focussed on 2 encoded features. In folktables, we selected a total of 4 target attributes as encoded features:
\begin{align*}
    &\text{\emph{PINCP}: Total income. Continuous}\\
    &\text{\emph{SCHL}: Educational attainment. Categorical (1-24)}\\
    &\text{\emph{JWMNP}: Tavel time to work. Continuous}\\
    &\text{\emph{ESR}: Employment status recode. Categorical (1-6)}\\
\end{align*}

\subsection{The results}

We restricted our experiments to values for which a one-sided error may occur. We denote by $\Delta_{\text{min}}$ the minimum value for which all inequalities \eqref{eq:gen_case_compar} are violated for a given pair of distributions, i.e., $\epsilon =1$. The corresponding $\Delta_{\text{max}}$ is equal to the supremum norm. 
We would then consider values in the range:
\[
\Delta_{\text{min}}<\Delta<\Delta_{\text{max}}\iff\epsilon\in(0,1)
\]

Figure \ref{fig:comp} compares the probability error in the subsampled supremum norm vs. the error of the Wasserstein-2 distance for different numbers of samples in a low-dimensional setting. Figure \ref{fig:comp_wass} shows a high standard deviation, while the equivalent shaded region in Figure \ref{fig:comp_point} is barely visible, although it increases as $\epsilon$ decreases. The probability of an error drops substantially with an increase in sample size. However, as we decrease $\epsilon$ (= increase $\Delta$), the error rate will increase, as the few pairs of histograms not fulfilling \eqref{eq:gen_case_compar} will be increasingly difficult to catch by subsampling.
In particular, where the fraction of these pairs of histograms is a thousandth ($\epsilon = 0.001$), we will get an error more than half the time. 

As discussed above, we expect the number of samples required for the proposed approach to scale as $O(\frac{n\log n}{\epsilon}\log\frac{n\log n}{\epsilon})$ in the number of encoded features, $n$, whereas for the distances studied in the literature, the curse of dimensionality occurs (cf. \cite{scholkpf16} and Proposition 10 in \cite{weed2019sharp}). Figure \ref{fig:folk} shows the expected behavior as the dimension increases by including additional encoded features: the error probability is only slightly increased for each of the corresponding values of $\epsilon.$

\begin{figure}[!htb]
    \centering
    \begin{subfigure}[b]{0.4\textwidth}
        \includegraphics[scale=0.46]{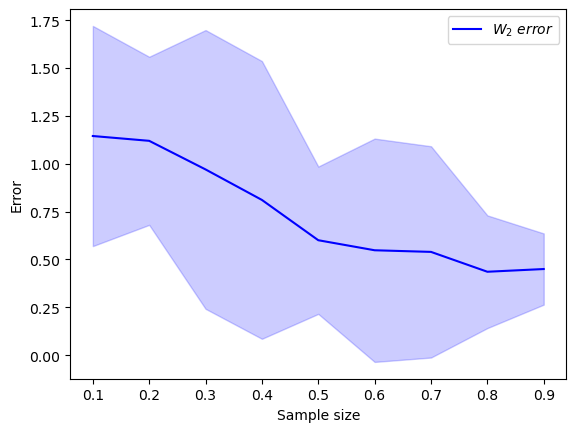}
        \caption{} 
        \label{fig:comp_wass}
    \end{subfigure}
    \begin{subfigure}[b]{0.4\textwidth}
        \includegraphics[scale=0.46]{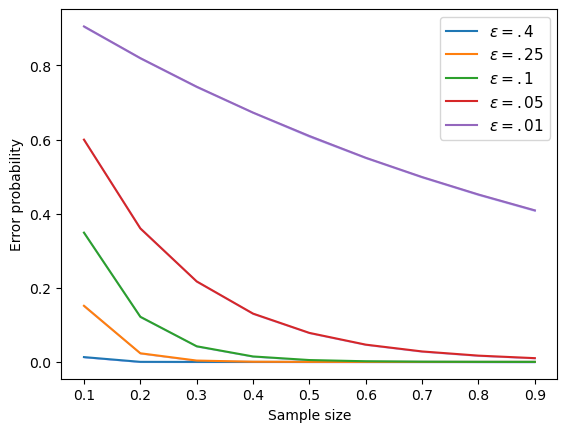}
        \caption{}
        \label{fig:comp_point}
    \end{subfigure}
    \caption{Probability of one-sided error for Wasserstein-2 and point-to-subspace distance in the supremum norm as a function of the sample size on the Adult dataset \cite{misc_adult_2}.}\label{fig:comp}
\end{figure}



\begin{figure}[t]
\centering
\includegraphics[scale=0.45]{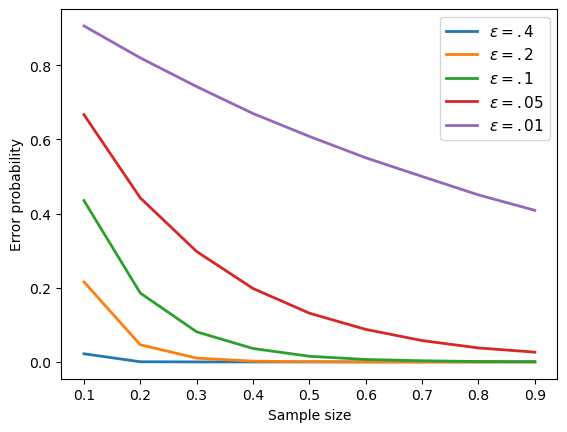}
\caption{Probability of one-sided error for point-to-subspace distance in the supremum norm as a function of the sample size on folktables \cite{ding2021retiring}.}
\label{fig:folk}
\end{figure}

\section{Conclusions}

We have presented probabilistically approximately correct (PAC) learnability of bias detection with respect to the supremum norm, with important applications in testing both the input (data quality) and output of AI systems. 

Overall, a substantially lower error can be obtained in the bias estimated using the supremum norm than in the bias estimated using the Wasserstein-2 norm, with a given budget in terms of sample complexity. Having a low error in estimating the bias, compared to the bias \emph{per se}, will be important in auditing the bias and any related judicial proceedings.

Moreover, within the PAC learning approach, one can control $\epsilon, \delta$ for a fixed test measure, e.g., the sample of cases available within the COMPAS data set, and consider the uncertainty in the estimate of the general population, e.g., census conducted every 10 years. 
The fixed size of the sample may be of importance in many applications, where the sample is obtained using freedom-of-information requests or requests made within AI-specific regulations.  
The fact that one can control $\epsilon, \delta$ also means that such an approach could be utilized in large language models, where traditional approaches based on optimal transport
\cite{vaserstein1969markov,vaserstein1969markov,dudley1969speed}, whose runtime scales superlinearly (often cubically) with the ambient dimension, may be challenging to apply.  

The results could be strengthened in a number of ways: one may wish to consider, for example, functions of bounded variation \cite{long1998sample}.

\nocite{}
\bibliographystyle{ieeetr}
\bibliography{refs.bib}
\end{document}